\documentclass[final,5p,times,twocolumn]{elsarticle}

\usepackage{amssymb}
\usepackage{amsmath}
\usepackage{amsthm}
\usepackage{subfigure}
\usepackage{hyperref}
\usepackage{algorithm}%
\usepackage{algorithmicx}%
\usepackage{booktabs}
\usepackage{multirow}
\newtheorem{theorem}{Theorem}

\newtheorem{lemma}{Lemma}
\usepackage{algpseudocode}%
\usepackage{listings}%

\journal{Neural Networks}

\begin{document}

\begin{frontmatter}



\title{Negative-Free Self-Supervised Gaussian Embedding of Graphs}


% \author[a]{Yunhui~Liu, Tieke~He, Tao~Zheng, Jianhua~Zhao} %% Author name

\author[1]{Yunhui~Liu}
\ead{lyhcloudy1225@gmail.com}
\author[1]{Tieke~He\corref{c}}
\ead{hetieke@gmail.com}
\author[1]{Tao~Zheng}
\author[1]{Jianhua~Zhao}

\cortext[c]{Corresponding author}

\affiliation[1]{organization={State Key Laboratory for Novel Software Technology},
            addressline={Nanjing University}, 
            city={Nanjing},
            postcode={210023}, 
            country={China}}

%% Abstract
\begin{abstract}
Graph Contrastive Learning (GCL) has recently emerged as a promising graph self-supervised learning framework for learning discriminative node representations without labels. The widely adopted objective function of GCL benefits from two key properties: \emph{alignment} and \emph{uniformity}, which align representations of positive node pairs while uniformly distributing all representations on the hypersphere. The uniformity property plays a critical role in preventing representation collapse and is achieved by pushing apart augmented views of different nodes (negative pairs). As such, existing GCL methods inherently rely on increasing the quantity and quality of negative samples, resulting in heavy computational demands, memory overhead, and potential class collision issues. In this study, we propose a negative-free objective to achieve uniformity, inspired by the fact that points distributed according to a normalized isotropic Gaussian are uniformly spread across the unit hypersphere. Therefore, we can minimize the distance between the distribution of learned representations and the isotropic Gaussian distribution to promote the uniformity of node representations. Our method also distinguishes itself from other approaches by eliminating the need for a parameterized mutual information estimator, an additional projector, asymmetric structures, and, crucially, negative samples. Extensive experiments over seven graph benchmarks demonstrate that our proposal achieves competitive performance with fewer parameters, shorter training times, and lower memory consumption compared to existing GCL methods. 
% The implementation code is available at \url{https://github.com/Cloudy1225/SSGE}.
\end{abstract}


% \begin{highlights}
% \item We propose a negative-free uniformity objective for self-supervised learning. This is inspired by the fact that points distributed according to a normalized isotropic Gaussian are uniformly spread across the unit hypersphere.
% \item The proposed objective liberates graph self-supervised learning from the reliance on negative samples and intricate components, including a parameterized mutual information estimator, an additional projector, or asymmetric architectures.

% \item Extensive experiments across seven graph datasets and two downstream tasks demonstrate that our model achieves competitive performance with fewer parameters, shorter training times, and lower memory consumption than existing graph contrastive learning methods. The implementation code is available at \url{https://github.com/Cloudy1225/SSGE}.
% \end{highlights}
    
%% Keywords

\end{frontmatter}



\section{Introduction}
One of the current main bottlenecks in graph machine learning is the dependence on heavy annotated training data. Learning representations on graphs without manual labels offers considerable advantages, as highlighted in recent literature \cite{GSSLSurvey, DGRLSurvey}. Against this backdrop, Graph Self-Supervised Learning (GSSL) has emerged as a pivotal methodology, addressing this critical need. GSSL demonstrates capabilities that are on par with, or in some cases, surpass those of supervised learning methods, as evidenced by studies such as \cite{DGI, GRACE, HomoGCL, BLNN, iGCL, S2T, TGC, wu2024high}. Central to the GSSL approach is pre-training, which utilizes meticulously designed pretext objectives. These objectives are task-agnostic, enabling the optimization process to yield representations that are not only general and meaningful but also transferable across various downstream applications.

As a prominent branch of the GSSL family, graph contrastive learning (GCL) \cite{GRACE, GCA, gCooL, COSTA, CGRA, InfoAdv} has demonstrated remarkable performance and garnered widespread interest. GCL methods focus on learning node representations by creating two augmented views of the input graph and maximizing the mutual information between the encoded representations. Wang and Isola \cite{UCL} offer an intuitive and theoretical understanding of contrastive learning, emphasizing \emph{alignment} and \emph{uniformity} on the hypersphere. \emph{Alignment}, i.e., pulling together positive pairs, ensures that samples forming positive pairs are mapped to nearby representations, thus rendering them invariant to irrelevant noise factors. However, relying solely on alignment could lead to complete collapse, where all representations converge to a single point \cite{SimSam}. In such cases, the learned representations may exhibit optimal alignment but fail to provide meaningful information for downstream tasks. \emph{Uniformity} assesses how representations are uniformly distributed across the hypersphere, with higher uniformity indicating that more information is preserved in the learned representations. Therefore, \emph{uniformity} plays a pivotal role in alleviating complete collapse and generating discriminative representations \cite{UCL}.

However, existing GCL methods achieve uniformity by pushing apart augmented views of different nodes (negative pairs). As such, they inherently rely on both the quantity and quality of negative samples. This reliance results in substantial computational and memory overhead, as well as class collision, where different samples from the same class are inaccurately deemed negative pairs, thereby impeding representation learning \cite{TACL, BGRL}. To address these issues, recent self-supervised methods \cite{CCASSG, BGRL, G-BT, AFGRL} have explored the prospect of learning without negative samples. Specifically, CCA-SSG \cite{CCASSG} and G-BT \cite{G-BT} learn augmentation-invariant information while introducing feature decorrelation to capture orthogonal features and prevent model collapse. However, they may not work well on datasets with low feature dimensions, as they essentially perform dimension reduction. BGRL \cite{BGRL} and AFGRL \cite{AFGRL} introduce an online network along with a target network to avoid collapse. But they require additional components, e.g., an exponential moving average (EMA) and Stop-Gradient, leading to a more intricate architecture.

Different from prior works on contrastive learning, we propose Negative-Free Self-Supervised Gaussian Embedding of Graphs (SSGE), a simple yet effective approach that introduces a new negative-free self-supervised learning objective while liberating the model from intricate designs. Following established practices, SSGE generates two views of an input graph through random augmentation and obtains node representations via a shared Graph Neural Network (GNN) encoder. Moreover, our contribution lies in proposing a negative-free self-supervised learning objective. Specifically, this new objective seeks to maximize the agreement between two augmented views of the same input (\emph{alignment}) while simultaneously minimizing the distance between the distribution of learned representations and the isotropic Gaussian distribution (\emph{uniformity}). Our motivation is grounded in the fact that the normalized isotropic Gaussian distributed points are uniformly distributed on the unit hypersphere. The proposed objective yields a simple and light model without depending on negative pairs \cite{GRACE, GCA, gCooL}, a parameterized mutual information estimator \cite{DGI, MVGRL}, an additional projector \cite{GRACE, GCA, BGRL, AFGRL}, or asymmetric architectures \cite{BGRL, AFGRL, GraphALU, BLNN}. Our model also works better on low-dimensional datasets than other simple models \cite{CCASSG, G-BT}. Extensive experiments on node classification and node clustering demonstrate that our model consistently achieves highly competitive performance. Furthermore, our method exhibits advantages such as fewer parameters, shorter training times, and lower memory consumption when compared to existing GCL methods. The implementation code is available at \url{https://github.com/Cloudy1225/SSGE}. To sum up, our contributions are as follows:
\begin{itemize}
    \item A negative-free uniformity objective is proposed, which is inspired by the fact that points distributed according to a normalized isotropic Gaussian are uniformly spread across the unit hypersphere.

    \item The proposed objective liberates the self-supervised learning model from the reliance on negative samples and intricate components, including a parameterized mutual information estimator, an additional projector, or asymmetric architectures.

    \item Extensive experiments across seven graph datasets and two downstream tasks demonstrate that our model achieves competitive performance with fewer parameters, shorter training times, and lower memory consumption than existing GCL methods.
\end{itemize}

\section{Related Works}
Recently, numerous research efforts have been devoted to multi-view graph self-supervised learning, which optimizes model parameters by ensuring consensus among multiple views derived from the same sample under different graph augmentations \cite{GSSLSurvey}. A crucial aspect of these methods is the prevention of trivial solutions, where all representations converge either to a constant point (i.e., complete collapse) or to a subspace (i.e., dimensional collapse). The existing methods can be broadly classified into two groups: contrastive and non-contrastive approaches, each delineated by its strategy for mitigating model collapse.

\textbf{Contrastive-based methods} typically follow the criterion of mutual information maximization \cite{DIM}, whose objective functions involve contrasting positive pairs with negative ones. Pioneering works, such as DGI \cite{DGI} and GMI \cite{GMI}, focus on unsupervised representation learning by maximizing mutual information between node-level representations and a graph summary vector. MVGRL \cite{MVGRL} proposes to learn both node-level and graph-level representations by performing node diffusion and contrasting node representations to augmented graph representation. GRACE \cite{GRACE} and its variants like GCA \cite{GCA}, gCooL \cite{gCooL}, and COSTA \cite{COSTA} learn node representations by pulling together the representations of the same node (positive pairs) in two augmented views while pushing away the representations of the other nodes (negative pairs) in two views. AUAR \cite{AUAR} aligns the representations of the node with itself and its cluster centroid while maximizing the distance between nodes and each cluster centroid. DirectCLR \cite{DirectCLR} directly optimizes the representation space without relying on a trainable projector to mitigate the dimensional collapse in contrastive learning.

\textbf{Non-contrastive methods} eliminate the use of negative samples and adopt different strategies to avoid collapsed solutions. Distillation-based methods BGRL \cite{BGRL}, AFGRL \cite{AFGRL} introduce an online network along with a target network, where the target network is updated with a moving average of the online network to avoid collapse. GraphALU \cite{GraphALU} further captures the uniformity by maximizing the distance between any nodes and a virtual center node. RGRL \cite{RGRL} learns augmentation-invariant relationship, which allows the node representations to vary as long as the relationship among the nodes is preserved. Feature decorrelation methods CCA-SSG \cite{CCASSG} and G-BT \cite{G-BT} rely on regularizing the empirical covariance matrix of the representations to capture orthogonal features and prevent dimensional collapse. W-MSE \cite{W-MSE} whitens and projects embeddings to the unit sphere before maximizing cosine similarity between positive samples.

Although these methods have demonstrated impressive performance, their reliance on intricate designs and architectures is noteworthy. For example, DGI \cite{DGI}, GMI \cite{GMI} and MVGRL \cite{MVGRL} rely on a parameterized Jensen-Shannon mutual information estimator \cite{f-GAN} for distinguishing positive node-graph pairs from negative ones. GRACE \cite{GRACE}, GCA \cite{GCA}, gCooL \cite{gCooL}, and COSTA \cite{COSTA} harness an additional MLP-projector to guarantee sufficient capacity. Moreover, they necessitate a substantial number of negative samples to prevent model collapse and learn discriminative representations, making them suffer seriously from heavy computation, memory overhead, and class collision \cite{TACL}. BGRL \cite{BGRL}, AFGRL \cite{AFGRL}, and GraphALU \cite{GraphALU} require asymmetric encoders, an exponential moving average and Stop-Gradient, to empirically avoid degenerated solutions, resulting in a more complex architecture. CCA-SSG \cite{CCASSG} and G-BT \cite{G-BT} may exhibit suboptimal performance on datasets where input data does not have a large feature dimension, as they are essentially performing dimension reduction. In contrast, our method aims to make the distribution of learned representations close to the isotropic Gaussian distribution to achieve uniformity while aligning the representations of two views from data augmentation.

\section{Preliminary}
\subsection{Problem Statement}
Let $\mathcal{G} = (\mathcal{V}, \mathcal{E})$ represent an attributed graph, where $\mathcal{V} = \{ v_1, v_2, \cdots, v_n\}$ and $\mathcal{E} \subseteq \mathcal V \times \mathcal V$ denote the node set and the edge set, respectively. The graph $\mathcal{G}$ is associated with a feature matrix $\boldsymbol{X} \in \mathbb{R}^{n \times p}$, where $\boldsymbol{x}_i \in \mathbb{R}^p$ represents the feature of $v_i$, and an adjacency matrix $\boldsymbol{A} \in \{ 0,1 \}^{n \times n}$, where $\boldsymbol{A}_{i,j} = 1$ if and only if $(v_i, v_j) \in \mathcal{E}$. In the self-supervised training setting, no task-specific labels are provided for $\mathcal{G}$. The primary objective is to learn an embedding function $f_\theta(\boldsymbol{A}, \boldsymbol{X})$ that transforms $\boldsymbol{X}$ to $\boldsymbol{Z}$, where $\boldsymbol{Z} \in \mathbb{R}^{n \times d}$ and $d \ll p$. The pre-trained representations aim to capture both attribute and structural information inherent in $\mathcal{G}$ and are easily transferable to various downstream tasks, such as node classification and node clustering.

\subsection{Graph Convolutional Network}
The Graph Convolutional Network (GCN) \cite{GCN} is one of the most popular graph neural networks. It is a layer-wise propagation rule-based model to learn the node representation $\boldsymbol{z}_i \in \mathbb{R}^d$ corresponding to node $v_i$. The formulation of the graph convolutional layer can be expressed as:
\begin{equation}
    \boldsymbol{Z}^{(l+1)} = \sigma \left(\boldsymbol{\hat{D}}^{-1/2} \boldsymbol{\hat{A}} \boldsymbol{\hat{D}}^{-1/2} \boldsymbol{Z}^{(l)} \boldsymbol{\Theta}^{(l)} \right),
\end{equation}
where $\boldsymbol{Z}^{(l+1)}$ denotes the node representations at the $l+1$ layer and $\boldsymbol{Z}^{(0)}$ represents the original attribute matrix of nodes. $\boldsymbol{\hat{D}}^{-1/2} \boldsymbol{\hat{A}} \boldsymbol{\hat{D}}^{-1/2}$ is a symmetric normalization of $\boldsymbol{A}$ with self-loop, $\boldsymbol{\hat{A}} = \boldsymbol{A}+\boldsymbol{I}$. $\boldsymbol{I}$ and $\boldsymbol{\hat{D}}$ are the identity matrix and the diagonal node degree matrix of $\boldsymbol{\hat{A}}$, respectively. Additionally, $\boldsymbol{\Theta}^{(l)}$ represents the weight matrix at the $l$-th layer, and $\sigma$ denotes the activation function. Consistent with previous works \cite{DGI, GRACE, CCASSG, gCooL, HomoGCL}, we adopt GCN as the foundational graph encoder.

\subsection{Wasserstein Distance}
Wasserstein distances are metrics quantifying the dissimilarity between probability distributions, drawing inspiration from the optimal transportation problem \cite{WGAN}. The $p$-Wasserstein distance is formulated as follows:
\begin{equation}
    \mathcal{W}_p(\mathbb{P}_r, \mathbb{P}_g) = \left( \inf_{\gamma \in \Pi (\mathbb{P}_r, \mathbb{P}_g)} \mathbb{E}_{(x,y) \sim \gamma} \left[ \|x-y\|^p \right] \right)^{\frac{1}{p}}, 
\end{equation}
where $\Pi (\mathbb{P}_r, \mathbb{P}_g)$ is the set of all joint distributions $\gamma(x,y)$ whose marginals are $\mathbb{P}_r$ and $\mathbb{P}_g$, respectively. The term $\gamma(x,y)$ intuitively indicates the amount of ``mass" requiring transportation from $x$ to $y$ for transforming the distribution $\mathbb{P}_r$ into $\mathbb{P}_g$. The Wasserstein distance thus represents the ``cost" associated with the optimal transport plan.

In the case of considering both distributions as multivariate Gaussian distributions, i.e., $\mathbb{P}_r=\mathcal{N}(\boldsymbol{\mu}_1, \boldsymbol{\Sigma}_1)$ and $\mathbb{P}_g=\mathcal{N}(\boldsymbol{\mu}_2, \boldsymbol{\Sigma}_2)$, with mean vectors
$\boldsymbol{\mu}_1$, $\boldsymbol{\mu}_2$, and covariance matrices $\boldsymbol{\Sigma}_1$, $\boldsymbol{\Sigma}_2$, respectively, the $2$-Wasserstein distance has a closed form expression given by
\begin{equation}
     \mathcal{W}_2^2(\mathbb{P}_r, \mathbb{P}_g) = \|\boldsymbol{\mu}_1-\boldsymbol{\mu}_2\|_2^2+\operatorname{Tr}\left( \boldsymbol{\Sigma}_1+\boldsymbol{\Sigma}_2 - 2\left( \boldsymbol{\Sigma}_2^\frac{1}{2} \boldsymbol{\Sigma}_1 \boldsymbol{\Sigma}_2^\frac{1}{2} \right)^\frac{1}{2} \right), \label{2-Wasserstein distance}
\end{equation}
where $\operatorname{Tr}(\cdot)$ denotes the trace of a matrix. This equation illustrates that the $2$-Wasserstein distance between two Gaussian distributions can be easily computed.

\section{Methodology}
\subsection{Model Framework}
Our model is simply constructed with three key components: 1) a random graph augmentation generator $\mathcal{T}$, 2) a GNN-based graph encoder symbolized as $f_\theta$, where $\theta$ representing its parameters, and 3) a Gaussian distribution guided objective function. Figure \ref{Fig: Overview} is an illustration of the proposed model.

\begin{figure*}[!ht]
\centerline{\includegraphics[width=1.\linewidth]{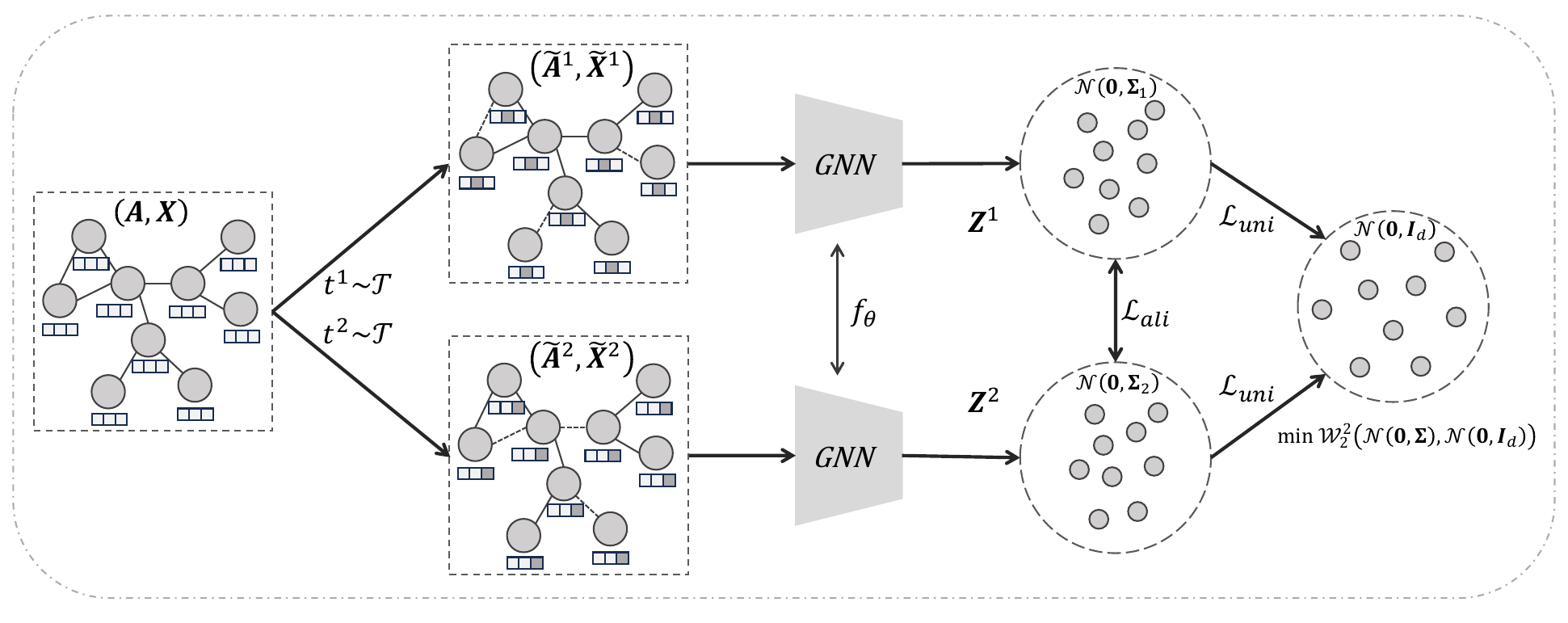}}\caption{Overview of our proposed graph self-supervised learning framework SSGE. For a given attributed graph $\left(\boldsymbol{A}, \boldsymbol{X}\right)$, we first generate two distinct views $(\boldsymbol{\widetilde{A^1}}, \boldsymbol{\widetilde{X^1}})$, $(\boldsymbol{\widetilde{A^2}}, \boldsymbol{\widetilde{X^2}})$ through random augmentations $t^1$, $t^2$. These two views are subsequently fed into a shared GNN encoder $f_\theta$ to extract batch-normalized node representations $\boldsymbol{Z}^1$, $\boldsymbol{Z}^2$. The \emph{alignment} loss $\mathcal{L}_{ali}$ and the \emph{uniformity} loss $\mathcal{L}_{uni}$ are applied on  $\boldsymbol{Z}^1$, $\boldsymbol{Z}^2$. Here, $\mathcal{L}_{uni}$ aims to minimize the 2-Wasserstein distance between the distribution of learned representations and the isotropic Gaussian distribution $\mathcal{N}\left( \boldsymbol{0}, \boldsymbol{I}_d\right)$.} \label{Fig: Overview}
\end{figure*}

\subsubsection{Graph Augmentation}
The augmentation of graph data is a critical component of graph contrastive learning, as it generates diverse graph views, resulting in more generalized representations that are robust against variance. In this study, we jointly adopt two widely utilized strategies, feature masking and edge dropping, to enhance graph attributes and topology information, respectively \cite{GRACE, GCA, CCASSG, BGRL, gCooL}. 

\textbf{Feature Masking.}\quad We randomly select a portion of the node features' dimensions and mask their elements with zeros. Formally, we first sample a random vector $\boldsymbol{\widetilde{m}} \in \{ 0, 1 \}^F$, where each dimension is drawn from a Bernoulli distribution with probability $1 - p_m$, i.e., $\widetilde{m}_i \sim \mathcal{B}(1 - p_m), \forall i$. Then, the masked node features $\widetilde {\boldsymbol{X}}$ are computed by $\parallel_{i=1}^N \boldsymbol{x}_i \odot \boldsymbol{\widetilde{m}}$, where $\odot$ denotes the Hadamard product and $\parallel$ represents the stack operation (i.e., concatenating a sequence of vectors along a new dimension).

\textbf{Edge Dropping.}\quad In addition to feature masking, we randomly drop a certain fraction of edges from the original graph. Formally, since we only remove existing edges, we first sample a random masking matrix $\boldsymbol{\widetilde{M}} \in \{ 0, 1 \}^{N \times N}$, with entries drawn from a Bernoulli distribution $\boldsymbol{\widetilde{M}}_{i,j} \sim \mathcal{B}(1 - p_d)$ if $\boldsymbol{A}_{i,j} = 1$ for the original graph, and $\boldsymbol{\widetilde{M}}_{i,j} = 0$ otherwise. Here, $p_d$ represents the probability of each edge being dropped. The corrupted adjacency matrix can then be computed as $\boldsymbol{\widetilde{A}} = \boldsymbol{A} \odot \boldsymbol{\widetilde{M}}$.

\subsubsection{Training and Inference}
During each training epoch, we first select two random augmentation functions, $t^1 \sim \mathcal{T}$ and $t^2 \sim \mathcal{T}$, where $\mathcal{T}$ is composed of all the possible graph transformation operations. Subsequently, two different views, $(\boldsymbol{\widetilde{A^1}}, \boldsymbol{\widetilde{X^1}}) = t^1(\boldsymbol{A},\boldsymbol{X})$ and $(\boldsymbol{\widetilde{A^2}}, \boldsymbol{\widetilde{X^2}})=t^2(\boldsymbol{A},\boldsymbol{X})$, are generated based on the sampled functions. These two augmented views are then fed into a shared encoder $f_\theta$ to extract the corresponding node representations: $\boldsymbol{Z}^1 = f_\theta(\boldsymbol{\widetilde{A^1}}, \boldsymbol{\widetilde{X^1}})$ and $\boldsymbol{Z}^2 = f_\theta(\boldsymbol{\widetilde{A^2}}, \boldsymbol{\widetilde{X^2}})$. To facilitate subsequent discussion, $\boldsymbol{Z}^1$ and $\boldsymbol{Z}^2$ are further batch-normalized so that each representation channel in which obeys a distribution with zero-mean and one-standard deviation. Finally, the model is optimized using some self-supervised objectives, such as InfoNCE \cite{InfoNCE} or our proposed objective defined in Eq. (\ref{Eq: Overall Objective}). After training, to obtain node representations for downstream tasks, the original graph $\mathcal{G} = (\boldsymbol{A}, \boldsymbol{X})$ is fed into the trained encoder $f_\theta$, yielding $\boldsymbol{Z} = f_\theta(\boldsymbol{A}, \boldsymbol{X})$.

% \subsubsection{Inference}
% To obtain node representations for downstream tasks, the original graph $\mathcal{G} = (\boldsymbol{A}, \boldsymbol{X})$ is fed into the trained encoder $f_\theta$, yielding $\boldsymbol{Z} = f_\theta(\boldsymbol{A}, \boldsymbol{X})$.

\subsection{Negative-Free Self-Supervised Loss}
In this subsection, we first analyze the weakness of contrastive-based self-supervised methods and then propose a negative-free self-supervised loss.

\subsubsection{Weakness of Graph Contrastive Learning}
In most graph contrastive learning methods \cite{GRACE, GCA, gCooL, COSTA}, both positive pairs and negative pairs are required for learning a model. For instance, the widely adopted InfoNCE \cite{InfoNCE} loss has the following formulation:
\begin{equation}
    \mathcal{L}_{InfoNCE} = \sum_{i=1}^n \underbrace{-\boldsymbol{z}_i^1 \cdot \boldsymbol{z}_i^2 / \tau}_{\text{alignment}} +  \underbrace{\log \left( \sum_j e^ {\boldsymbol{z}_i^1 \cdot \boldsymbol{z}_j^2 / \tau} \right)}_{\text{uniformity}}, \label{Eq: InfoNCE}
\end{equation}
where $\boldsymbol{z}_i^1 \in \mathbb{R}^d$ and $\boldsymbol{z}_i^2 \in \mathbb{R}^d$ are the ($\ell_2$-normalized) representations of two views of the same sample $i$, and $\tau$ is the temperature hyperparameter. Minimizing Eq. (\ref{Eq: InfoNCE}) is equivalent to maximizing the cosine similarity of two views of the same sample (\emph{alignment}) and meanwhile minimizing the cosine similarity of two views of different samples (\emph{uniformity}). Intuitively, the \emph{alignment} term makes the positive pairs close to each other, while the \emph{uniformity} term distributes all samples roughly uniformly on the hypersphere $\mathcal{S}^{d-1}$ \cite{UCL}. Based on this analysis, Wang and Isola \cite{UCL} propose a uniformity metric by utilizing the logarithm of the average pairwise Gaussian potential:
\begin{equation}
    \mathcal{L}_{uniform} = \log \frac{1}{n(n-1)/2} \sum_{i=2}^n \sum_{j=1}^{i-1} e^{(-t\|\boldsymbol{z}_i - \boldsymbol{z}_j\|_2^2)}, t>0.
\end{equation}
This uniformity metric is expected to be both asymptotically correct (i.e., the distribution optimizing this metric should converge to uniform distribution) and empirically reasonable with a finite number of samples. However, both the uniformity term in InfoNCE and this pairwise Gaussian potential based uniformity metric inherently rely on the large number and high quality of negative samples. This reliance results in substantial computational and memory overhead, as well as class collision, where diverse samples from the same class are inaccurately deemed negative pairs, thereby impeding representation learning for classification \cite{TACL, BGRL}. In this work, we introduce a new negative-free uniformity objective derived from hyperspherical uniform distribution.

\subsubsection{Uniformity from Isotropic Gaussian Distribution}
We first show that zero-mean isotropic (equal-variance) Gaussian distributed vectors (after normalized to norm $1$) are uniformly distributed over the unit hypersphere with the following theorem.

\begin{theorem}[Hyperspherical Uniformity \cite{HU}]\label{sphereuniform}
The normalized vector of Gaussian variables is uniformly distributed on the hypersphere. Formally, let $z_1,z_2,\cdots,z_d\sim \mathcal{N}(0,1)$ and be independent. Then the vector
\begin{equation}
    \boldsymbol{z}=\bigg{[} \frac{z_1}{r},\frac{z_2}{r},\cdots,\frac{z_d}{r} \bigg{]}
\end{equation}
follows the uniform distribution on $\mathcal{S}^{d-1}$, where $r=\sqrt{z_1^2+z_2^2+\cdots+z_d^2}$ is a normalization factor.
\end{theorem}

\begin{proof}
A random variable has distribution $\mathcal{N}(0,1)$ if it has the density function
\begin{equation}
f(x)=\frac{1}{\sqrt{2\pi}}e^{-\frac{1}{2}x^2}.
\end{equation}
A $d$-dimensional random vector $\boldsymbol{z}$ has distribution $\mathcal{N}(\boldsymbol{0},\boldsymbol{I}_d)$ if the components are independent and have distribution $\mathcal{N}(0,1)$ each. Then the density of $\boldsymbol{z}$ is given by
\begin{equation}
f(\boldsymbol{z})=\frac{1}{(\sqrt{2\pi})^d}e^{-\frac{1}{2}\langle \boldsymbol{z},\boldsymbol{z}\rangle},
\end{equation}
where $\langle \cdot, \cdot \rangle$ denotes the inner product. Then we introduce the following lemma (Lemma~\ref{lemma_sphereuniform}) about the orthogonal-invariance of the Gaussian distribution.

\begin{lemma}\label{lemma_sphereuniform}
Let $\boldsymbol{z}$ be a $d$-dimensional random vector with distribution $\mathcal{N}(\boldsymbol{0},\boldsymbol{I}_d)$ and $\boldsymbol{U}\in\mathbb{R}^{d\times d}$ be an orthogonal matrix ($\boldsymbol{U}\boldsymbol{U}^\top =\boldsymbol{U}^\top\boldsymbol{U}=\boldsymbol{I}_d$). Then $\boldsymbol{Y}=\boldsymbol{U}\boldsymbol{z}$ also has the distribution of $\mathcal{N}(\boldsymbol{0},\boldsymbol{I}_d)$.
\end{lemma}

\begin{proof}
For any measurable set $\mathcal{A}\subset\mathbb{R}^d$, we have that
\begin{equation}
\begin{aligned}
    P(\boldsymbol{Y}\in \mathcal{A})&= P(\boldsymbol{Z}\in \boldsymbol{U}^\top \mathcal{A})\\
    &=\int_{\boldsymbol{U}^\top \mathcal{A}}\frac{1}{(\sqrt{2\pi})^d} e^{-\frac{1}{2}\langle \boldsymbol{z},\boldsymbol{z}\rangle}\\
    &=\int_\mathcal{A}\frac{1}{(\sqrt{2\pi})^d}e^{-\frac{1}{2}\langle \boldsymbol{U}\boldsymbol{z},\boldsymbol{U}\boldsymbol{z}\rangle}\\
    &=\int_\mathcal{A}\frac{1}{(\sqrt{2\pi})^d}e^{-\frac{1}{2}\langle \boldsymbol{z},\boldsymbol{z}\rangle}
\end{aligned}
\end{equation}
because of orthogonality of $\boldsymbol{U}$. Therefore the lemma holds.
\end{proof}

Since any rotation is just a multiplication with some orthogonal matrix, we know that normally distributed random vectors are invariant to rotation. As a result, generating $\boldsymbol{z}\in\mathbb{R}^d$ with distribution $\mathcal{N}(\boldsymbol{0},\boldsymbol{I}_d)$ and then projecting it onto the hypersphere $\mathcal{S}^{d-1}$ produces random vectors $\frac{\boldsymbol{z}}{\|\boldsymbol{z}\|_2}$ that are uniformly distributed on the hypersphere. Therefore the theorem holds.
\end{proof}

The aforementioned theorem establishes an equivalence between the hyperspherical uniform distribution and the normalized isotropic Gaussian distribution. Consequently, we can utilize the distance between the distribution of learned representations and the isotropic Gaussian distribution as a measure of uniformity. Specifically, given learned representations $\boldsymbol{Z} \in \mathbb{R}^{n \times d}$, we first apply batch normalization to ensure that each representation channel follows a distribution with zero mean and unit variance. To simplify computation, we adopt a Gaussian hypothesis for the normalized representations and assume they conform to $\mathcal{N}(\boldsymbol{0}, \boldsymbol{\Sigma})$, where $\boldsymbol{\Sigma} = \frac{1}{n-1}\boldsymbol{Z}^\top \boldsymbol{Z}$, and the on-diagonal elements are $1$. Building upon this assumption, we employ the 2-Wasserstein distance, as defined in Eq. (\ref{2-Wasserstein distance}), to quantify the disparity between the learned representation distribution $\mathcal{N}(\boldsymbol{0}, \boldsymbol{\Sigma})$ and the target isotropic Gaussian distribution $\mathcal{N}(\boldsymbol{0}, \boldsymbol{I}_d)$ as the uniformity objective:
\begin{equation}
\begin{aligned}
     \mathcal{U}(\boldsymbol{Z})
     &=\mathcal{W}_2^2 \left(\mathcal{N}(\boldsymbol{0}, \boldsymbol{\Sigma}), \mathcal{N}(\boldsymbol{0}, \boldsymbol{I}_d) \right) \\
     &=\|\boldsymbol{0}-\boldsymbol{0}\|_2^2+\operatorname{Tr}\left(\boldsymbol{\Sigma}+\boldsymbol{I}_d - 2\left( \boldsymbol{I}_d ^\frac{1}{2} \boldsymbol{\Sigma} \boldsymbol{I}_d^\frac{1}{2} \right)^\frac{1}{2} \right) \\
     &= 2d - 2\operatorname{Tr} \left(\boldsymbol{\Sigma}^\frac{1}{2} \right). \label{Eq: Wasserstein Uniformity}
\end{aligned}
\end{equation}
Minimizing this uniformity objective promotes the proximity of the learned representation distribution to the isotropic Gaussian distribution, thereby enhancing the uniformity of the learned representations. Consequently, for the representations $\boldsymbol{Z}^1$ and $\boldsymbol{Z}^2$ associated with the two generated graph views during training, the expression for the \emph{uniformity} loss is:
\begin{equation}
\begin{aligned}
     \mathcal{L}_{uni}&=\frac{1}{2}(\mathcal{U}(\boldsymbol{Z}^1)+\mathcal{U}(\boldsymbol{Z}^2)).\\
     &=2d-\operatorname{Tr}\left(\boldsymbol{\Sigma}_1^\frac{1}{2} \right)-\operatorname{Tr}\left(\boldsymbol{\Sigma}_2^\frac{1}{2} \right),     \label{Eq: Uniformity}
\end{aligned}     
\end{equation}
where $\boldsymbol{\Sigma}_1 = \frac{1}{n-1}{\boldsymbol{Z}^1}^\top \boldsymbol{Z}^1$ and $\boldsymbol{\Sigma}_2 = \frac{1}{n-1}{\boldsymbol{Z}^2}^\top \boldsymbol{Z}^2$ are estimated covariance matrices of the two view representations.

\subsubsection{Combining with the View Consistency Prior}
Similar to the alignment term in contrastive learning objectives, we try to maximize the correlation between two views by minimizing the Euclidean distance between representations derived from one sample:
\begin{equation}
    \mathcal{L}_{ali} = \frac{1}{n}\| \boldsymbol{Z}^1 - \boldsymbol{Z}^2 \|_2^2.\label{Eq: Euclidean distance}
\end{equation}
Since $\boldsymbol{Z}^1$ and $\boldsymbol{Z}^2$ are bath-normalized, $\| \boldsymbol{Z}^1 - \boldsymbol{Z}^2 \|_2^2 $ can be computed by:
\begin{equation}
\begin{aligned}
    \| \boldsymbol{Z}^1 - \boldsymbol{Z}^2 \|_2^2
    &= \sum_j^d \sum_i^n \left(z^1_{i,j}-z^2_{i,j}\right)^2\\
    &= \sum_j^d \sum_i^n \left( \left((z^1_{i,j}\right)^2 + \left((z^2_{i,j}\right)^2 - 2z^1_{i,j}z^2_{i,j} \right)\\
    &= \sum_j^d \left(2 - 2 \left(\boldsymbol{Z}^1_{\cdot, j}\right)^\top \boldsymbol{Z}^2_{\cdot, j}\right)\\
    &= \left(2d-2\operatorname{Tr} \left( {\boldsymbol{Z}^1}^\top \boldsymbol{Z}^2 \right) \right).
\end{aligned}
\end{equation}
For simplicity, we adopt the following equivalent form of Eq. (\ref{Eq: Euclidean distance}) as the implementation of our \emph{alignment} loss:
\begin{equation}
    \mathcal{L}_{ali} = \frac{1}{n} \operatorname{Tr}\left({\boldsymbol{Z}^1}^\top \boldsymbol{Z}^2 \right). \label{Eq: Alignment}
\end{equation}
The alignment loss promotes the mapping of different augmentation views of the same sample to nearby representations, thus enabling the model to learn representations invariant to unneeded noise factors.

\subsubsection{Overall Objective}
Combining the \emph{uniformity} loss and the \emph{alignment} loss, we formulate our overall objective as follows:
\begin{equation}
    \mathcal{L} = \mathcal{L}_{ali} + \lambda \mathcal{L}_{uni}, \label{Eq: Overall Objective}
\end{equation}
where $\lambda$ is a non-negative hyperparameter that balances the contributions of the two terms.

\noindent\textbf{Complexity Analysis.}\quad
Consider a graph with $n$ nodes, and each node is embedded into a $d$-dimensional vector. The computation of the alignment term $\mathcal{L}_{ali}$ requires $\mathcal{O}(n)$ time and $\mathcal{O}(n)$ space. The computation of $\boldsymbol{\Sigma}^{\frac{1}{2}}$ in the uniformity term $\mathcal{L}_{uni}$ is implemented using eigenvalue decomposition, which takes $\mathcal{O}(d^3)$ time and $\mathcal{O}(d^2)$ space, with typically $d \ll n$. In contrast, contrastive learning methods \cite{GRACE, GCA, COSTA, HomoGCL} treat two views of the same node as positive pairs and views of different nodes as negative pairs, incurring $\mathcal{O}(n^2)$ time and $\mathcal{O}(n^2)$ space. As a result, our negative-free method holds more promise for handling large-scale graphs without incurring prohibitively high time and space costs compared to contrastive learning methods.

\begin{algorithm}
    \caption{The overall procedure of SSGE}
    \label{Alg: algorithm}
    \textbf{Input}: $\mathcal{G} = (\boldsymbol{A}, \boldsymbol{X})$\\
    \textbf{Parameter}: Trade-off $\lambda$, Augmentor $\mathcal{T}$\\
    \textbf{Output}: The graph encoder $f_\theta$
    \begin{algorithmic}[1] %[1] enables line numbers
        \State Initialize model parameters;
        \While{not converge}
        \State Sample two augmentation functions $t^1, t^2 \sim \mathcal{T}$;
        \State Generate two augmented graphs via $t^1(\boldsymbol{A},\boldsymbol{X})$ and $t^2(\boldsymbol{A},\boldsymbol{X})$;
        \State Obtain batch-normalized node representations $\boldsymbol{Z}^1$ and $\boldsymbol{Z}^2$ using $f_\theta$;
        \State Compute the uniformity loss via Eq. \ref{Eq: Uniformity};
        \State Compute the alignment loss via Eq. \ref{Eq: Alignment};
        \State Update the parameters of $f_\theta$ by minimizing Eq. \ref{Eq: Overall Objective};
        \EndWhile
        \State \textbf{return} $f_\theta$.
    \end{algorithmic}
\end{algorithm}

\noindent\textbf{Advantages over Peer Works}\quad
Here we conduct a systematic comparison with previous graph self-supervised learning methods, including DGI \cite{DGI}, MVGRL \cite{MVGRL}, GRACE \cite{GRACE}, GCA \cite{GCA}, BGRL \cite{BGRL}, AFGRL \cite{AFGRL}, CCA-SSG \cite{CCASSG}, and G-BT \cite{G-BT}. In brief, our SSGE stands out by mitigating reliance on negative samples and intricate components. Specifically, DGI and MVGRL require a parameterized Jensen-Shannon estimator for approximating mutual information between two views. MVGRL also introduces asymmetric architectures by employing two different GNNs for the input graph and the diffusion graph, respectively. On the other hand, GRACE and GCA employ an additional MLP-projector followed by an InfoNCE mutual information estimator, which relies on a substantial number of negative samples, leading to heavy computation, memory overhead, and class collision. BGRL and AFGRL require an asymmetric encoder architecture that incorporates Exponential Moving Average (EMA), Stop-Gradient, and an additional projector. CCA-SSG and G-BT may demonstrate suboptimal performance on datasets with smaller feature dimensions, as they essentially perform dimension reduction. In contrast, our framework requires no additional components except a single GNN encoder, our Gaussian distribution-guided objective is negative-free, and our model can also work well on low-dimensional datasets.

\subsubsection{In-depth Analysis}

We recognize that the optimal solution for the final loss functions of CCA-SSG, G-BT, and our SSGE is $\boldsymbol{\Sigma} = \boldsymbol{I}_d$. Despite such similarity, we demonstrate the superiority of our objective function theoretically as follows.

We begin by briefly discussing the phenomenon of dimensional collapse in representation learning. Dimensional collapse occurs when the embedding occupies a subspace rather than the entire embedding space, which is indicated by one or more zero singular values of the embedding matrix  (in other words, one or more zero eigenvalues $\lambda_i$ of the covariance matrix $\boldsymbol{\Sigma}$) \cite{DirectCLR}. To address dimensional collapse, G-BT, CCA-SSG, and our proposed SSGE decorrelate the covariance matrix through a regularization term. Specifically, G-BT and CCA-SSG minimize $||\boldsymbol{\Sigma}-\boldsymbol{I}_d||_F^2$ to encourage the covariance matrix to approximate the identity matrix. In contrast, our SSGE minimizes the distance between the distribution of learned representations and the isotropic Gaussian distribution, which can be simplified as $-\operatorname{Tr}(\boldsymbol{\boldsymbol{\Sigma}}^{\frac{1}{2}})$ according to Eq. (\ref{Eq: Wasserstein Uniformity}).

Note that $||\boldsymbol{\Sigma}-\boldsymbol{I}_d||_F^2 = \sum_{i=1}^d (\lambda_i-1)^2$, $-\operatorname{Tr}(\boldsymbol{\Sigma}^{1/2})= -\sum_{i=1}^d \sqrt{\lambda_i}$, and $\sum_{i=1}^d \lambda_i = ||\boldsymbol{Z}|_F^2 = d$ since $\boldsymbol{Z}$ is batch-normalized. We can reformulate these regularization terms using eigenvalues $\lambda_i$: for G-BT and CCA-SSG, the problem becomes $\min \sum_{i=1}^d (\lambda_i - 1)^2$ subject to $\sum_{i=1}^d \lambda_i = d$, while for SSGE, it is $\min -\sum_{i=1}^d \sqrt{\lambda_i}$ under the same constraint. The optimal solution for both optimization problems is $\lambda_i = 1$ for each $i$.

However, when optimizing model parameters using SGD, G-BT and CCA-SSG may inadequately regularize zero eigenvalues compared to our SSGE. Specifically, the gradient of $||\boldsymbol{\Sigma}-\boldsymbol{I}_d||_F^2 = \sum_{i=1}^d (\lambda_i-1)^2$ with respect to the eigenvalue $\lambda_i$ is $\nabla_{\lambda_i} = 2(\lambda_i-1)$. This regularization term imposes a finite penalty on singular values $\lambda_i$ approaching 0, potentially leading to dimensional collapse if other loss terms offer greater rewards than the penalty.

In contrast, the gradient of $-\operatorname{Tr}(\boldsymbol{\Sigma}^{1/2})= -\sum_{i=1}^d \sqrt{\lambda_i}$ with respect to $\lambda_i$ is $\nabla_{\lambda_i} = -\frac{1}{2\sqrt{\lambda_i}}$, resulting in a significantly larger gradient as $\lambda_i$ approaches 0. Thus, the optimization process will naturally steer away from nearly infinite penalties toward full rank $\boldsymbol{\Sigma}$, avoiding dimensional collapse more effectively.

\section{Experiments}
In this section, we design the experiments to evaluate our proposed SSGE and answer the following research questions. \textbf{RQ1:} Does SSGE outperform existing baseline methods on node classification and node clustering? \textbf{RQ2:} Is SSGE more efficient than graph contrastive learning baselines? \textbf{RQ3:} How does each component of SSGE benefit the performance? \textbf{RQ4:} How to intuitively understand the role of our proposed uniformity objective? \textbf{RQ5:} Are learned node representations uniformly distributed on the hypersphere? \textbf{RQ6:} Is SSGE sensitive to hyperparameters?

\subsection{Experiment Setup}
\subsubsection{Datasets}
We adopt seven publicly available real-world benchmark datasets, including three citation networks Cora, CiteSeer, and PubMed,  one reference network WikiCS, one co-purchase network Computer, one co-authorship network CoauthorCS, and one large-scale citation network ArXiv to conduct the experiments throughout the paper. The statistics of the datasets are provided in Table \ref{Tab: Dataset statistics}. We give their detailed descriptions as follows:

\begin{table}[h]
    \begin{center}
    \setlength{\tabcolsep}{3.5pt}
    {\caption{Dataset statistics.}\label{Tab: Dataset statistics}}
    \vspace{-0.2cm}
    \begin{tabular}{lccccc}
    \toprule
    Dataset    & Type           & \#Nodes    & \#Edges   & \#Feats  & \#Cls \\
    \midrule 
    Cora       & citation       & 2,708      & 10,556    & 1,433       & 7          \\
    CiteSeer   & citation       & 3,327      & 9,228     & 3,703       & 6          \\
    PubMed     & citation       & 19,717     & 88,651    & 500         & 3          \\
    WikiCS     & reference      & 11,701     & 431,726   & 300         & 10         \\
    Computer   & co-purchase    & 13,752     & 491,722   & 767         & 10         \\
    CoauthorCS & co-authorship  & 18,333     & 163,788   & 6,805       & 15         \\
    ArXiv      & citation       & 169,343    & 2,315,598 & 128         & 40         \\
    \bottomrule
    \end{tabular}
    \end{center}
\end{table}

\begin{itemize}
    \item \textbf{Cora}, \textbf{CiteSeer}, \textbf{PubMed} \cite{GCN} are three well-known citation network datasets, in which nodes represent publications and edges indicate their citations. Each node in Cora and CiteSeer is described by a 0/1-valued word vector indicating the absence/presence of the corresponding word from the dictionary, while each node in PubMed is described by a TF/IDF weighted word vector from the dictionary. All nodes are labeled based on the respective paper subjects.

    \item \textbf{WikiCS} \cite{WikiCS} is a reference network constructed from Wikipedia. It comprises nodes corresponding to articles in the field of Computer Science, where edges are derived from hyperlinks. The dataset includes 10 distinct classes representing various branches within the field. The node features are computed as the average GloVe word embeddings of the respective articles.

    \item \textbf{Computer} \cite{Amazon} is a network constructed from Amazon's co-purchase relationships. Nodes represent goods, and edges indicate frequent co-purchases between goods. The node features are represented by bag-of-words encoding of product reviews, and class labels are assigned based on the respective product categories.

    \item \textbf{CoauthorCS} \cite{Amazon} is a co-authorship network based on the Microsoft Academic Graph. Here, nodes are authors, that are connected by an edge if they co-authored a paper; node features represent paper keywords for each author’s papers, and class labels indicate the most active fields of study for each author.

    \item \textbf{ArXiv} \cite{Amazon} is a citation network between most Computer Science arXiv papers indexed by Microsoft Academic Graph, where nodes represent papers and edges represent citation relations. Each node is described by a $128$-dimensional feature vector obtained by averaging the skip-gram word embeddings in its title and abstract. The nodes are categorized by their related research area.

\end{itemize}

\subsubsection{Baselines}
To verify the effectiveness of our proposed model, our evaluation includes a comprehensive comparison of SSGE with $22$ baseline methods. We divide all baseline models into the following four categories:
\begin{itemize}
    \item Supervised methods: MLP and graph neural networks GCN \cite{GCN} and GAT \cite{GAT};

    \item Traditional unsupervised models: random-walk based graph embedding methods DeepWalk \cite{DeepWalk} and Node2Vec \cite{Node2Vec}, graph auto-encoders GAE and VGAE \cite{VGAE};

    \item Contrastive learning models: GMI \cite{GMI}, DGI \cite{DGI}, GGD \cite{GGD}, MVGRL \cite{MVGRL}, GRACE \cite{GRACE}, GCA \cite{GCA}, MERIT \cite{MERIT}, gCooL \cite{gCooL}, COSTA \cite{COSTA}, and HomoGCL \cite{HomoGCL};

    \item Non-contrastive self-supervised learning models: BGRL \cite{BGRL}, AFGRL \cite{AFGRL}, RGRL \cite{RGRL}, CCA-SSG \cite{CCASSG}, and G-BT \cite{G-BT}.
\end{itemize}

\begin{table}[h]
\begin{center}
\setlength{\tabcolsep}{4pt}
{\caption{Hyperparameter Specifications.}\label{Tab: Hyperparameter Specifications}}
\vspace{-0.2cm}
\begin{tabular}{lccccc}
\toprule
Dataset   & $p_{d}, p_{m}$  & lr, wd        & $\lambda$  & \#hid\_units  & \#epochs \\
\midrule
Cora      & 0.3, 0.1        & 1e-3, 1e-5    & 0.1        & 256-256       & 80       \\
CiteSeer  & 0.4, 0.0        & 1e-3, 1e-5    & 0.05       & 512           & 20       \\
PubMed    & 0.3, 0.5        & 1e-3, 1e-5    & 0.6        & 512-256       & 100      \\
WikiCS    & 0.8, 0.1        & 1e-2, 1e-6    & 0.5        & 256-256       & 50       \\
Computer  & 0.1, 0.3        & 1e-3, 1e-5    & 1.0        & 512-512       & 120      \\
CoauthorCS        & 1.0, 0.2        & 1e-3, 1e-5    & 0.05       & 512-512       & 80       \\
ArXiv     & 0.5, 0.3        & 1e-2, 1e-6    & 3.0        & 512-512       & 400      \\
\bottomrule
\end{tabular}
\end{center}
\end{table}

\begin{table*}[!ht]
    \begin{center}
    \setlength{\tabcolsep}{4pt}
    {\caption{Node classification results measured by accuracy along with standard deviations. The \emph{Input} column illustrates the data used in the training stage, and $\boldsymbol{Y}$ denotes labels. `-' means the method is out-of-memory under a full-graph training setting.}\label{Tab: Node Classification}}
    \vspace{0.2cm}
    \begin{tabular}{c|llccccccc}
        \toprule
        & Method    & Input      & Cora & CiteSeer & PubMed  & WikiCS & Computer & CoauthorCS & ArXiv  \\
        \midrule
    \multirow{3}{*}{\rotatebox{00}{Supervised}}
        & MLP       & $\boldsymbol{X}, \boldsymbol{Y}$                    & 57.8±0.2 & 54.2±0.1 & 72.8±0.2 & 71.98±0.00 & 73.81±0.00 & 90.37±0.00 & 55.50±0.23 \\
        & GCN       & $\boldsymbol{A}, \boldsymbol{X}, \boldsymbol{Y}$    & 81.5±0.4 & 70.2±0.4 & 79.0±0.2 & 77.19±0.12 & 86.51±0.54 & 93.03±0.31 & 71.74±0.29 \\
        & GAT       & $\boldsymbol{A}, \boldsymbol{X}, \boldsymbol{Y}$    & 83.0±0.7 & 72.5±0.7 & 79.0±0.3 & 77.65±0.11 & 86.93±0.29 & 92.31±0.24 & 72.10±0.13 \\
        \midrule
    \multirow{4}{*}{\rotatebox{0}{Traditional}}
        & DeepWalk  & $\boldsymbol{A}$                                    & 68.5±0.5 & 49.8±0.2 & 66.2±0.7 & 74.35±0.06 & 85.68±0.06 & 84.61±0.22 & ~ \\
        & Node2Vec  & $\boldsymbol{A}$                                    & 70.1±0.4 & 49.8±0.3 & 69.8±0.7 & 71.79±0.05 & 84.39±0.08 & 85.08±0.03 & 70.07±0.13 \\
        & GAE       & $\boldsymbol{A}, \boldsymbol{X}$                    & 71.5±0.4 & 65.8±0.4 & 72.1±0.5 & 77.87±0.53 & 85.27±0.19 & 90.01±0.71 & - \\
        & VGAE      & $\boldsymbol{A}, \boldsymbol{X}$                    & 73.0±0.3 & 68.3±0.4 & 75.8±0.2 & 77.87±0.53 & 86.37±0.21 & 92.11±0.09 & - \\
        \midrule
    \multirow{10}{*}{\rotatebox{0}{Contrastive}}
        & GMI       & $\boldsymbol{A}, \boldsymbol{X}$                    & 83.0±0.3 & 72.4±0.1 & 79.9±0.2 & 74.85±0.08 & 82.21±0.31 & 88.78±0.12 & - \\
        & DGI       & $\boldsymbol{A}, \boldsymbol{X}$                    & 82.3±0.6 & 71.8±0.7 & 76.8±0.6 & 78.25±0.56 & 83.95±0.47 & 92.15±0.63 & 70.34±0.16 \\
        & GGD       & $\boldsymbol{A}, \boldsymbol{X}$                    & 83.5±0.6 & 73.0±0.6 & 81.0±0.8 & 78.62±0.47 & 88.12±0.62 & 92.30±0.23 & 71.20±0.20 \\
        & MVGRL     & $\boldsymbol{A}, \boldsymbol{X}$                    & 83.5±0.4 & 73.3±0.5 & 80.1±0.7 & 77.57±0.46 & 87.52±0.11 & 92.11±0.12 & - \\
        & GRACE     & $\boldsymbol{A}, \boldsymbol{X}$                    & 81.9±0.4 & 71.3±0.3 & 80.1±0.2 & 78.64±0.33 & 88.29±0.11 & 92.17±0.04 & - \\
        & GCA       & $\boldsymbol{A}, \boldsymbol{X}$                    & 81.7±0.3 & 71.1±0.4 & 79.5±0.5 & 78.35±0.05 & 87.85±0.31 & 93.10±0.01 & - \\
        & MERIT     & $\boldsymbol{A}, \boldsymbol{X}$                    & 83.1±0.6 & 73.5±0.7 & 80.1±0.4 & 78.35±0.05 & 88.01±0.12 & 92.51±0.14 & - \\
        & gCooL     & $\boldsymbol{A}, \boldsymbol{X}$                    & 82.8±0.5 & 72.0±0.3 & 80.2±0.4 & 78.74±0.04 & 88.67±0.10 & 92.75±0.01 & - \\
        & COSTA     & $\boldsymbol{A}, \boldsymbol{X}$                    & 82.2±0.2 & 70.7±0.5 & 80.4±0.3 & 78.82±0.12 & 88.32±0.03 & 92.94±0.10 & - \\
        & HomoGCL   & $\boldsymbol{A}, \boldsymbol{X}$                    & \textbf{84.1±0.5} & 72.3±0.7 & 81.1±0.3 & 78.26±0.32 & 88.46±0.20 & 92.28±0.03 & - \\
        \midrule
    \multirow{6}{*}{\rotatebox{0}{Non-contrastive}}
        & BGRL      & $\boldsymbol{A}, \boldsymbol{X}$                    & 82.7±0.6 & 71.1±0.8 & 79.6±0.5 & 78.41±0.09 & 87.89±0.10 & 92.72±0.03 & 71.44±0.12 \\
        & AFGRL     & $\boldsymbol{A}, \boldsymbol{X}$                    & 79.8±0.2 & 69.4±0.2 & 80.0±0.1 & 77.62±0.49 & 88.12±0.27 & 93.07±0.17 &- \\
        & RGRL      & $\boldsymbol{A}, \boldsymbol{X}$                    & 83.5±0.7 & 71.5±0.9 & 79.8±0.3 & 78.78±0.64 & 88.45±0.52 & 92.94±0.13 & 71.49±0.08 \\
        & CCA-SSG   & $\boldsymbol{A}, \boldsymbol{X}$                    & 83.9±0.4 & 73.0±0.3 & 80.7±0.4 & 77.92±0.65 & 88.76±0.36 & 93.01±0.20 & 71.21±0.20 \\
        & G-BT      & $\boldsymbol{A}, \boldsymbol{X}$                    & 83.6±0.0 & 72.9±0.4 & 80.4±0.1 & 76.83±0.73 & 87.93±0.36 & 92.91±0.25 & 71.12±0.18 \\
        & SSGE      & $\boldsymbol{A}, \boldsymbol{X}$                    & 83.9±0.3 & \textbf{74.1±0.3} & \textbf{81.6±0.1} & \textbf{79.18±0.57} & \textbf{89.05±0.58} & \textbf{93.46±0.45} & \textbf{71.62±0.19} \\
        \bottomrule
    \end{tabular}
    \end{center}
\end{table*}

\subsubsection{Evaluation Protocol}
We evaluate the performance of SSGE on two downstream tasks, i.e., node classification and node clustering. Following previous work \cite{HomoGCL}, We first train the model in an unsupervised manner. Then we freeze the parameters of the encoder to generate node representations for downstream tasks. For node classification, we use the learned representations to train and test a simple logistic regression classifier on public splits for Cora, CiteSeer, PubMed, WikiCS, and ArXiv, and ten 1:1:8 train/validation/test random splits for Computer and CoauthorCS, as they have no publicly accessible splits. We train the model for ten runs and report the performance in terms of accuracy. For node clustering, we train a $k$-means model on the learned representations ten times, where the number of clusters is set to the number of classes for each dataset. We measure the clustering performance in terms of two prevalent metrics Normalized Mutual Information (NMI) score: $\operatorname{NMI}=2I(\hat{\boldsymbol{Y}};\boldsymbol{Y})/[H(\hat{\boldsymbol{Y}})+H(\boldsymbol{Y})]$, where $\hat{\boldsymbol{Y}}$ and $\boldsymbol{Y}$ being the predicted cluster indexes and class labels respectively, $I(\cdot)$ being the mutual information, and $H(\cdot)$ being the entropy; and Adjusted Rand Index (ARI): $\operatorname{ARI}=\text{RI}-\mathbb{E}[\text{RI}]/(\max \{ \text{RI} \} - \mathbb{E}[\text{RI}])$, where RI being the Rand Index.

\subsubsection{Implementation Details}
We implement our model with PyTorch. All experiments are conducted on a V100 GPU with 32 GB of memory. The graph encoder $f_\theta$ is specified as a standard two-layer GCN model with the ELU activation \cite{ELU} for all the datasets except CiteSeer (where we empirically find that a one-layer GCN is better). During training, we employ the Adam SGD optimizer \cite{Adam} with a learning rate and weight decay of either (1e-3, 1e-5) or (1e-2, 1e-6). The augmentation function set $\mathcal{T}$ is controlled by the edge dropping ratio $p_d$ and the feature masking ratio $p_m$. We utilize the processed versions of all datasets provided by the Deep Graph Library \cite{DGL}. Detailed hyperparameters can be found in Table \ref{Tab: Hyperparameter Specifications}. The implementation code is available at \url{https://github.com/Cloudy1225/SSGE}.

\begin{table*}[!ht]
  \begin{center}
  {\caption{Node clustering results measured by NMI and ARI along with standard deviations.}
  \label{Tab: Node Clustering}}
  \vspace{0.2cm}
  \begin{tabular}{l|cc|cc|cc}
  \toprule
  Dataset    & \multicolumn{2}{c|}{Cora}  & \multicolumn{2}{c|}{CiteSeer}  & \multicolumn{2}{c}{PubMed} \\ \midrule
    Metric   & NMI        & ARI        & NMI        & ARI        & NMI        & ARI                \\ \midrule
    $k$-means & 15.44±3.83 & 9.49±2.01 & 20.66±2.83 & 16.80±3.02 & 31.34±0.15 & 28.12±0.03 \\ 
    SC & 46.07±0.99 & 34.39±0.98 & 23.95±0.53 & 18.48±0.42 & 28.75±0.00 & 30.34±0.00 \\ 
    VGAE & 52.48±1.33 & 43.99±2.34 & 34.46±0.92 & 32.65±0.92 & 27.16±1.45 & 26.32±1.15 \\  \midrule
    DGI & 55.82±0.60 & 48.91±1.42 & 41.16±0.54 & 39.78±0.74 & 25.27±0.02 & 24.06±0.03 \\ 
    MVGRL & 56.30±0.27 & 50.28±0.40 & 43.47±0.08 & 44.09±0.09 & 27.07±0.00 & 24.53±0.00 \\ 
    GRACE & 55.82±0.60 & 48.91±1.42 & 39.07±0.07 & 40.38±0.08 & 30.44±0.02 & 30.62±0.01 \\ 
    GCA & 52.82±1.11 & 46.29±1.80 & 41.08±0.28 & 41.72±0.30 & 31.62±0.08 & 30.76±0.17 \\ 
    gCooL & 50.25±1.08 & 44.95±1.74 & 41.67±0.41 & 42.66±0.47 & 33.14±0.02 & 31.93±0.01 \\ 
    HomoGCL & 57.87±1.47 & 53.68±1.77 & 40.32±0.07 & 40.10±0.09 & 27.67±1.78 & 25.59±0.84 \\ \midrule
    BGRL & 55.38±0.49 & 47.10±0.35 & 38.95±0.33 & 38.81±0.12 & 28.43±0.10 & 24.81±0.08 \\ 
    CCA-SSG & 56.51±1.49 & 50.77±3.39 & 43.69±0.24 & 44.26±0.23 & 29.61±0.01 & 25.81±0.01 \\ 
    G-BT & 55.54±1.19 & 48.39±2.44 & 42.78±0.25 & 43.89±0.21 & 30.12±0.03 & 29.32±0.02 \\ 
    SSGE & \textbf{60.58±0.25} & \textbf{56.96±0.34} & \textbf{45.27±0.33} & \textbf{46.87±0.37} & \textbf{33.42±0.12} & \textbf{32.05±0.11} \\ 
  \bottomrule
  \end{tabular}
  \end{center}
\end{table*}

\begin{table*}[!ht]
\begin{center}
\setlength{\tabcolsep}{3.8pt}
{\caption{Comparison of the number of parameters, training time for achieving the best performance, and the memory cost of different methods on Cora, CiteSeer, PubMed, and Computer.}
\label{Tab: Efficiency Comparison}}
\begin{tabular}{l|ccc|ccc|ccc|ccc}
    \toprule
 Dataset      & \multicolumn{3}{c|}{Cora} & \multicolumn{3}{c}{CiteSeer} & \multicolumn{3}{c|}{PubMed} & \multicolumn{3}{c}{Computer} \\ \midrule
              & \#Paras & Time & Mem  & \#Paras & Time & Mem  & \#Paras & Time & Mem  & \#Paras & Time & Mem \\ \midrule	
 DGI          & 1,260K & 4.06s & 570MB & 2,422K & 4.51s & 710MB & 782K & 10.24s & 1,024MB & 919K & 13.50s & 926MB \\
 MVGRL        & 1,731K & 17.73s & 3,838MB & 4,055K & 15.48s & 7,386MB & 775K & 83.64s & 4,622MB & 1,049K & 124.96s & 4,660MB \\
 GRACE        & 564K & 10.57s & 748MB & 2,159K & 7.59s & 1,018MB & 326K & 396.97s & 12,744MB & 394K & 267.72s & 6,630MB \\
 HomoGCL      & 866K & 3.95s & 1,014MB & 2,028K & 6.01s & 1,478MB & 388K & 167.97s & 22,010MB & 525K & 92.35s & 11,788MB \\
 SSGE         & 433K & 3.41s & 580MB & 1,896K & 2.44s & 800MB & 388K & 7.90s & 1,256MB & 656K & 15.42s & 1,234MB \\
    \bottomrule
\end{tabular}
\end{center}
\end{table*}

\subsection{Main Results}
\subsubsection{Performance Comparison (\textbf{RQ1})}
The experimental results of node classification on seven datasets are shown in Table \ref{Tab: Node Classification}. As we can see, SSGE outperforms all self-supervised baselines on six out of seven datasets, despite its simple architecture. On Cora, SSGE achieves competitive results as the most powerful baseline HomoGCL. It is worth mentioning that we empirically find that on CoauthorCS, a pure two-layer MLP encoder is better than GNN models. This might be because the graph-structured information is much less informative than the node features, presumably providing harmful signals for classification. We also evaluate the node clustering performance on the three citation networks Cora, CiteSeer, and PubMed. As shown in Table \ref{Tab: Node Clustering}, SSGE can always yield significant improvements over other methods, especially on Cora with a 2.7\% gain for NMI and a 3.3\% gain for ARI. These results clearly demonstrate the effectiveness of our method.

\subsubsection{Efficiency Comparison (\textbf{RQ2})}
We conduct a comparative analysis with several graph contrastive learning methods regarding the number of model parameters, training time, and memory costs on datasets including Cora, CiteSeer, PubMed, and Computer. Table \ref{Tab: Efficiency Comparison} summarizes all indicators of various methods. Overall, SSGE has fewer parameters, shorter training time, and fewer memory costs in most cases. This is because our method does not rely on additional projection heads, parameterized mutual information estimator, and negative samples, which contribute to increased computational load, additional parameters, and storage requirements. Besides, the short training time potentially indicates the fast convergence of our algorithm. Despite its simplicity and efficiency, our method achieves even better (or competitive) performance.

\subsection{In-Depth Analysis}
\subsubsection{Effectiveness of Alignment/Uniformity Terms (\textbf{RQ3})}
To verify the effects of each loss component, we conduct ablation studies with varying combinations of the alignment term and the uniformity term. We assess their impact on node classification across four datasets, and the results are presented in Table \ref{Tab: Ablation Study}. It is observed that only using the alignment term will lead to a performance drop instead of completely collapsed solutions since node representations are batch-normalized to have zero-mean and one-standard deviation. On the other hand, optimizing only the uniformity term yields unsatisfactory results, as the model learns nothing meaningful but Gaussian distributed representations. These results highlight the effectiveness of incorporating the alignment term and the uniformity term to learn discriminative node representations.

\begin{table}[h]
\centering
    \renewcommand\arraystretch{1.2}
    \begin{tabular}{l|cccc}
    \toprule
    Variants & Cora & CiteSeer & PubMed & WikiCS\\
    \midrule
    $\mathcal{L}_{ali}$      & 79.7±0.1 & 71.9±0.2 & 78.1±0.2 & 77.02±0.40  \\
    $\mathcal{L}_{uni}$      & 48.7±0.6 & 30.1±0.7 & 51.4±0.5 & 76.47±0.60  \\
    $\mathcal{L}_{ali}+\mathcal{L}_{uni}$     & 83.9±0.3 & 74.1±0.3 & 81.6±0.1 & 79.18±0.57   \\
    \bottomrule
    \end{tabular}
\caption{Ablation study of node classification accuracy (\%) on the key components of SSGE.}
\label{Tab: Ablation Study}
\end{table}

\subsubsection{Visualization of Feature Covariance Matrix (\textbf{RQ4})}
To gain a visual insight into the role of the uniformity term, we present visualizations depicting the normalized covariance matrix $\boldsymbol{\Sigma}$ of learned representations under various settings on Cora. As depicted in Figure \ref{Fig: COV Visualization1}, the off-diagonal elements of the covariance matrix approach $1$ when the uniformity term is not considered. This suggests that various dimensions of the representation matrix are coupled together, signifying the occurrence of dimensional collapse (all the embeddings lie in a line). From Figure \ref{Fig: COV Visualization2} and Figure \ref{Fig: COV Visualization3}, we observe that the proposed uniformity objective effectively decorrelates diverse representation dimensions, thereby mitigating the issue of dimensional collapse.

\begin{figure}[h]
    \centering
    \subfigure[$\mathcal{L}_{ali}$]{\includegraphics[width=0.32\linewidth]{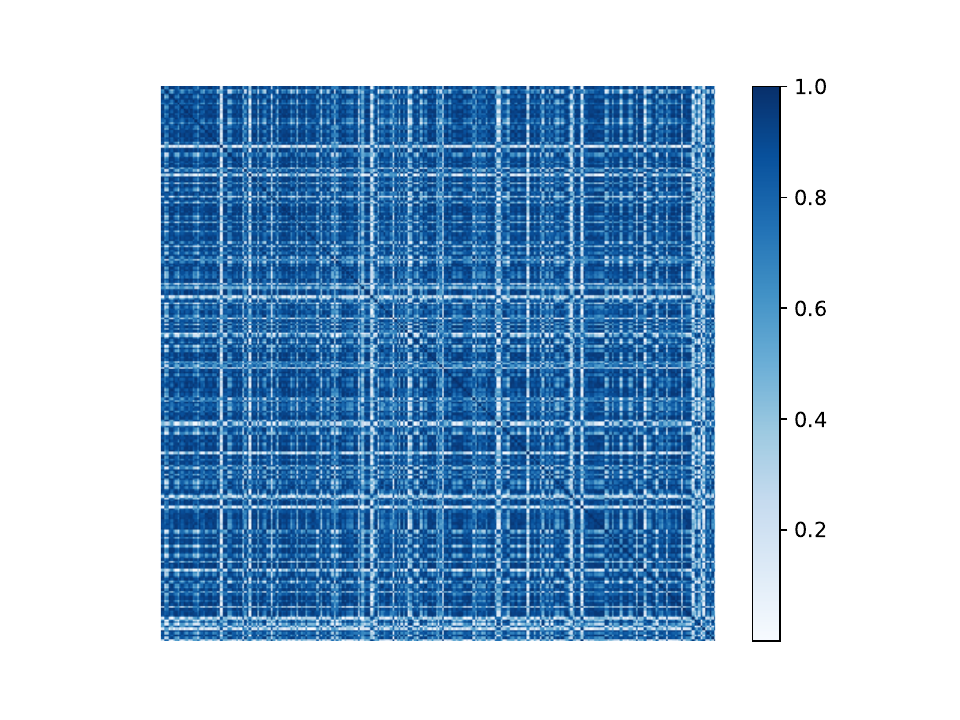}\label{Fig: COV Visualization1}}   
    \subfigure[$\mathcal{L}_{uni}$]{\includegraphics[width=0.32\linewidth]{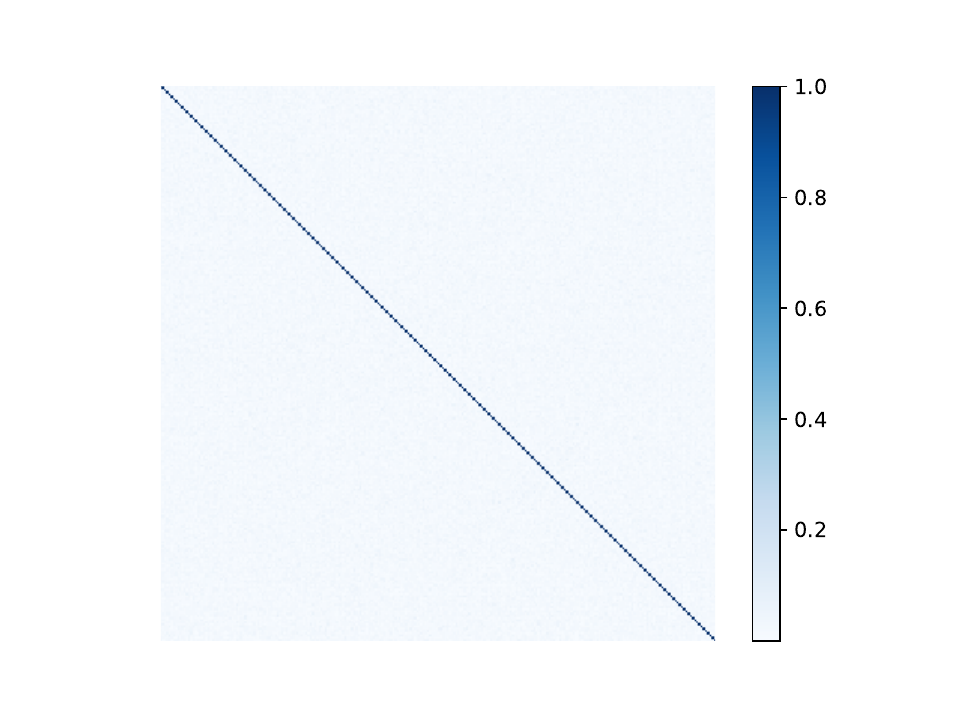}\label{Fig: COV Visualization2}}  
    \subfigure[$\mathcal{L}_{ali}+\mathcal{L}_{uni}$]{\includegraphics[width=0.32\linewidth]{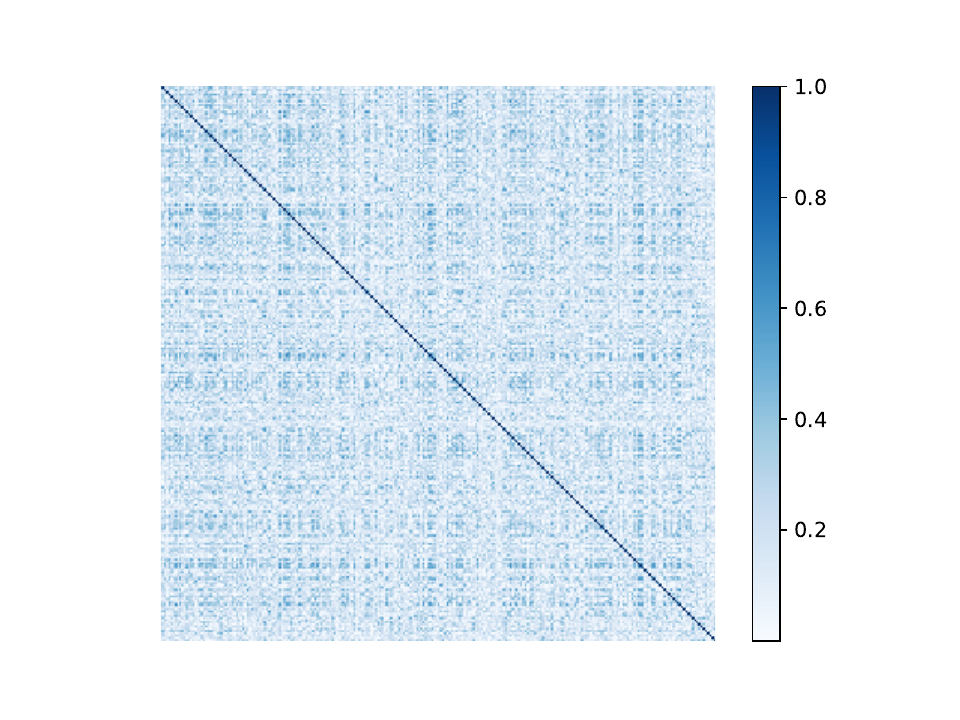}\label{Fig: COV Visualization3}}
    \caption{ Visualization of the covariance matrix (absolute value) of learned representations on Cora.}
    \label{Fig: COV Visualization}
\end{figure}

\subsubsection{Visualization of Feature Distribution on $\mathcal{S}^1$ (\textbf{RQ5})}
To verify the uniformity of node representations learned by SSGE, we visualize feature distributions of all classes and selected specific classes on $\mathcal{S}^1$ using Gaussian kernel density estimation in $\mathbb{R}^2$. Figure \ref{Fig: Uniformity Visualization1} and Figure \ref{Fig: Uniformity Visualization2} summarize the resulting distributions of learned representations on Cora and CiteSeer, respectively. As can be seen, representations learned by SSGE are both aligned (having low intra-class distances) and uniform (evenly distributed on $\mathcal{S}^1$).

\begin{figure}[h]
    \centering
    \subfigure[All Classes]{\includegraphics[width=0.24\linewidth]{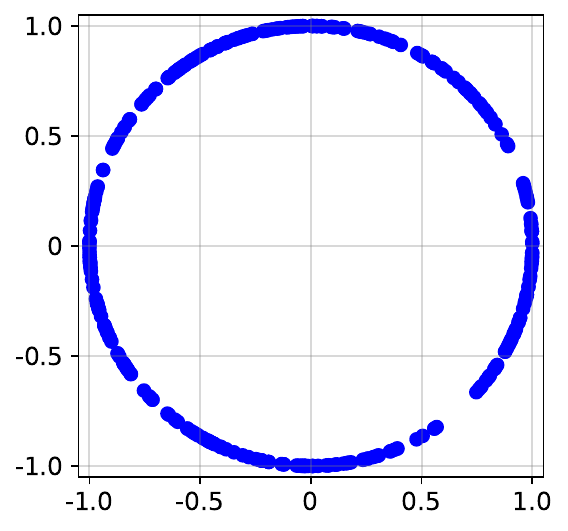}}   
    \subfigure[Class 0]{\includegraphics[width=0.24\linewidth]{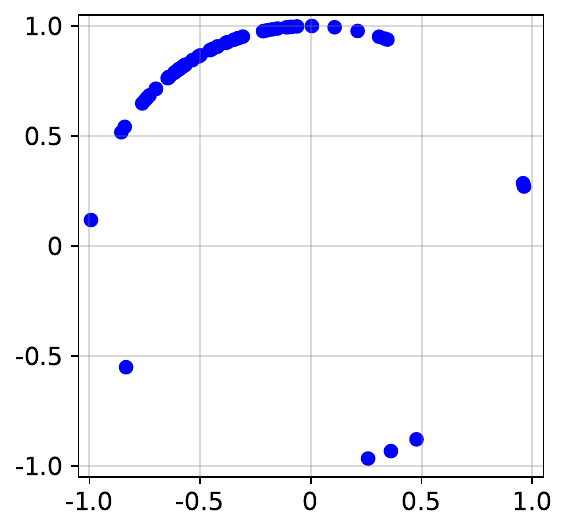}}  
    \subfigure[Class 1]{\includegraphics[width=0.24\linewidth]{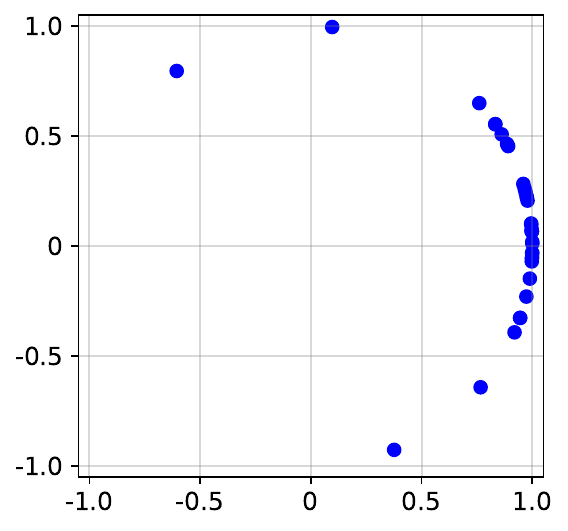}}
    \subfigure[Class 2]{\includegraphics[width=0.24\linewidth]{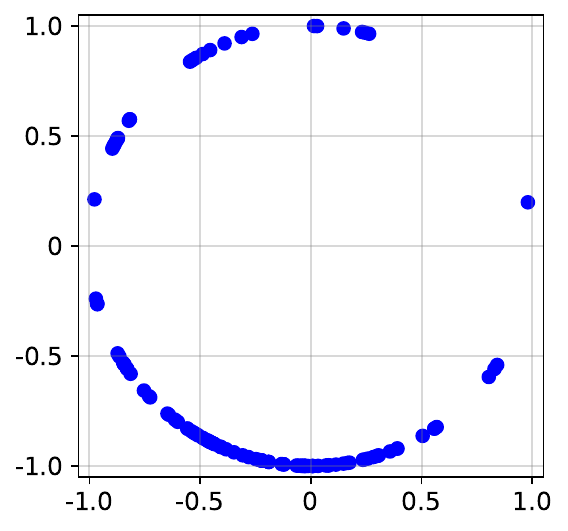}}
    \caption{Visualizing the alignment and uniformity of node representations on Cora.}
    \label{Fig: Uniformity Visualization1}
\end{figure}

\begin{figure}[h]
    \centering
    \subfigure[All Classes]{\includegraphics[width=0.24\linewidth]{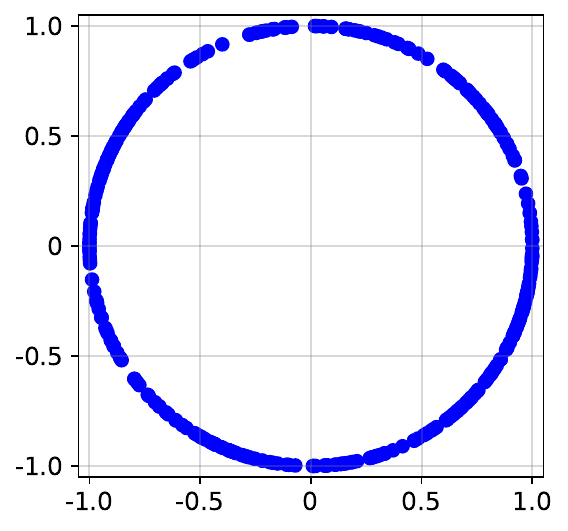}}   
    \subfigure[Class 0]{\includegraphics[width=0.24\linewidth]{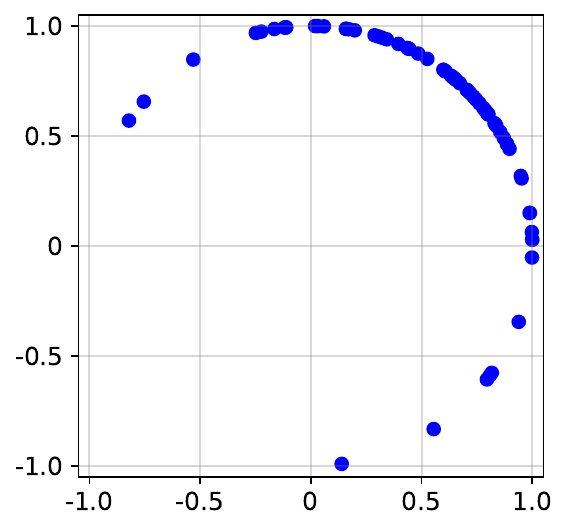}}  
    \subfigure[Class 1]{\includegraphics[width=0.24\linewidth]{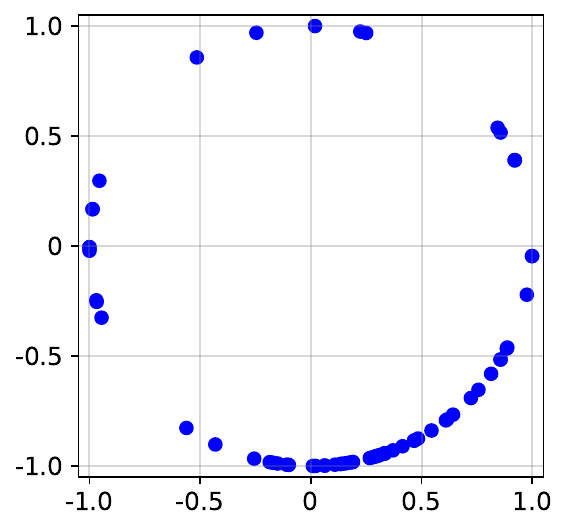}}
    \subfigure[Class 2]{\includegraphics[width=0.24\linewidth]{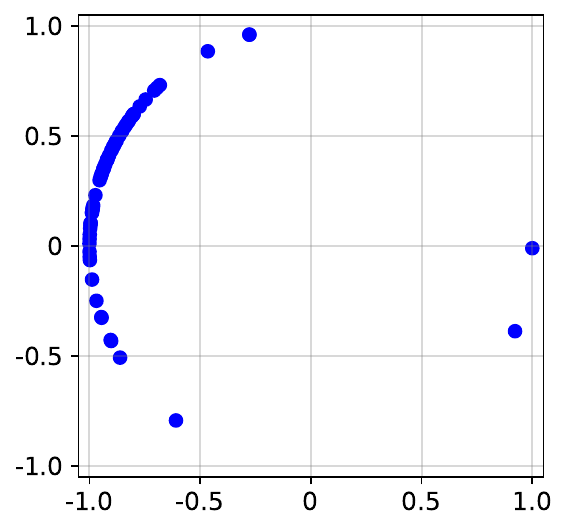}}
    \caption{Visualizing the alignment and uniformity of node representations on CiteSeer.}
    \label{Fig: Uniformity Visualization2}
\end{figure}

\subsection{Hyperparameter Analysis (\textbf{RQ6})}
\subsubsection{Impact of Uniformity Intensity}
We investigate how the intensity of uniformity influences the performance as the trade-off hyperparameter $\lambda$ is varied. Figure \ref{Fig: param_lam} depicts the variation in classification accuracy with different values of $\lambda$ across Cora, WikiCS, Computer, and CoauthorCS. It can be observed that at the beginning, increasing $\lambda$ enhances performance, but excessively large values of $\lambda$ result in a significant degradation of performance. When $\lambda$ is too small, the uniformity term fails to fully leverage its role in promoting uniform representations. Conversely, when $\lambda$ is excessively large, placing too much emphasis on uniformity while neglecting alignment leads to meaningless representations. To choose the value of $\lambda$, we recommend conducting a grid search or a sensitivity analysis on a subset of datasets during the initial experimentation phase. This approach can help identify a suitable range for $\lambda$ that balances performance across different contexts.

% \begin{figure}[!ht]
% \centerline{\includegraphics[width=1.\linewidth]{images/lam.pdf}}
% \caption{Impact of the trade-off $\lambda$. } \label{Fig: param_lam}
% \end{figure}

\begin{figure}[!ht]
    \centering
    \subfigure[trade-off $\lambda$]{\includegraphics[width=0.495\linewidth]{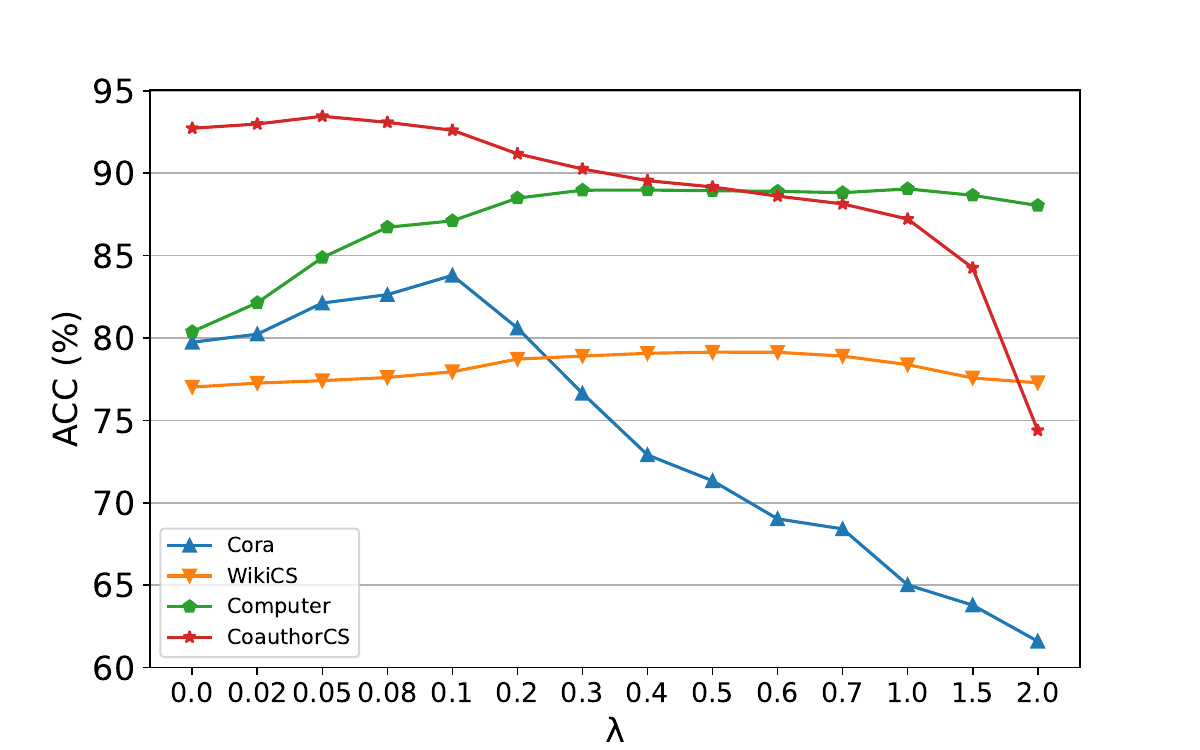}\label{Fig: param_lam}}   
    \subfigure[representation dimension]{\includegraphics[width=0.495\linewidth]{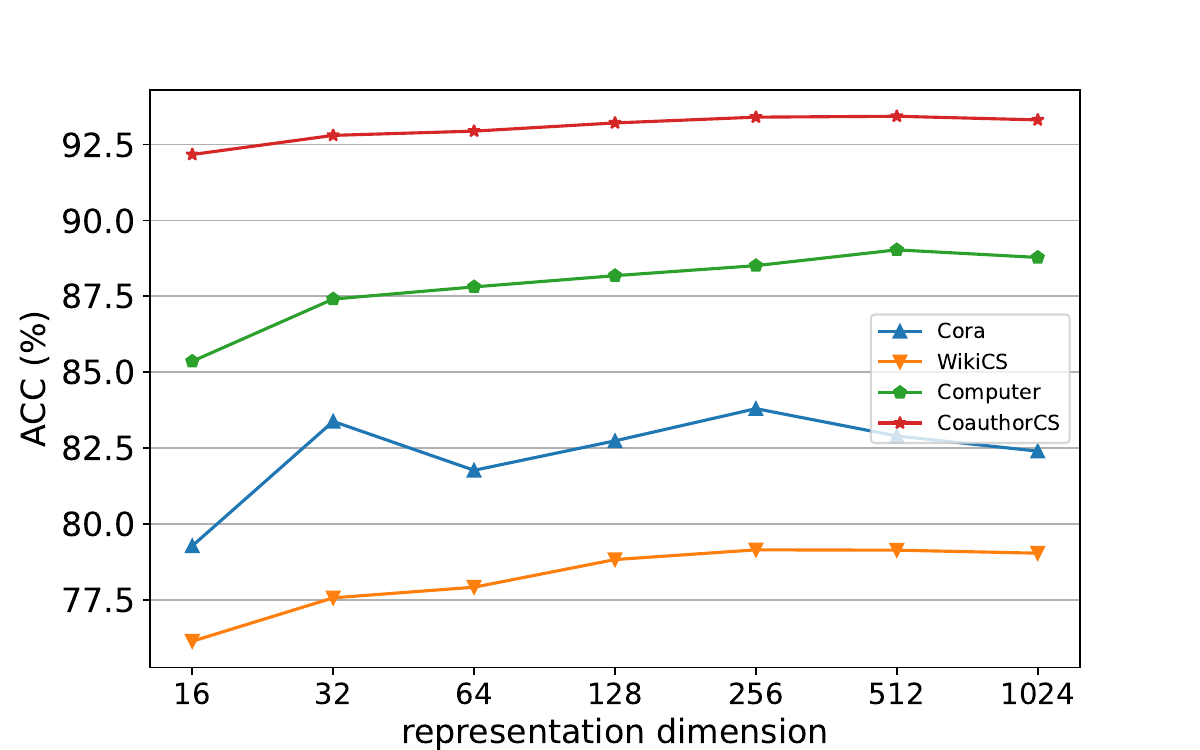}\label{Fig: param_dim}} 
    \caption{Impact of the trade-off $\lambda$ and representation dimension.}
    % \label{Fig: param_pdpm}
\end{figure}

\subsubsection{Impact of Representation Dimension}
We conduct experiments to explore the impact of varying the representation dimension on performance, as illustrated in Figure \ref{Fig: param_lam}. Similar to other self-supervised methods such as DGI \cite{DGI}, GRACE \cite{GRACE}, and BGRL \cite{BGRL}, our method achieves optimal performance with an appropriately large dimension (usually $256$ or $512$), while a too large dimension $1,024$ results in a slight decrease in performance. In comparison to feature decorrelation methods like CCA-SSG \cite{CCASSG} and G-BT \cite{G-BT}, our method can also work well with a representation dimension smaller than the input dimension, on low-dimensional datasets such as WikiCS and PubMed.

\subsubsection{Impact of Augmentation Intensity}
We further conduct a sensitivity analysis on the augmentation intensity by investigating the effects of different combinations of the feature masking ratio $p_m$ and the edge dropping ratio $p_d$. The results of node classification on Cora, CiteSeer, and WikiCS are presented in Figure \ref{Fig: param_pdpm}. Overall, our method demonstrates robustness to augmentation intensity: within an appropriate range of $p_m$ and $p_d$, our approach consistently achieves competitive results.

\begin{figure}[!ht]
    \centering
    \subfigure[Cora]{\includegraphics[width=0.32\linewidth]{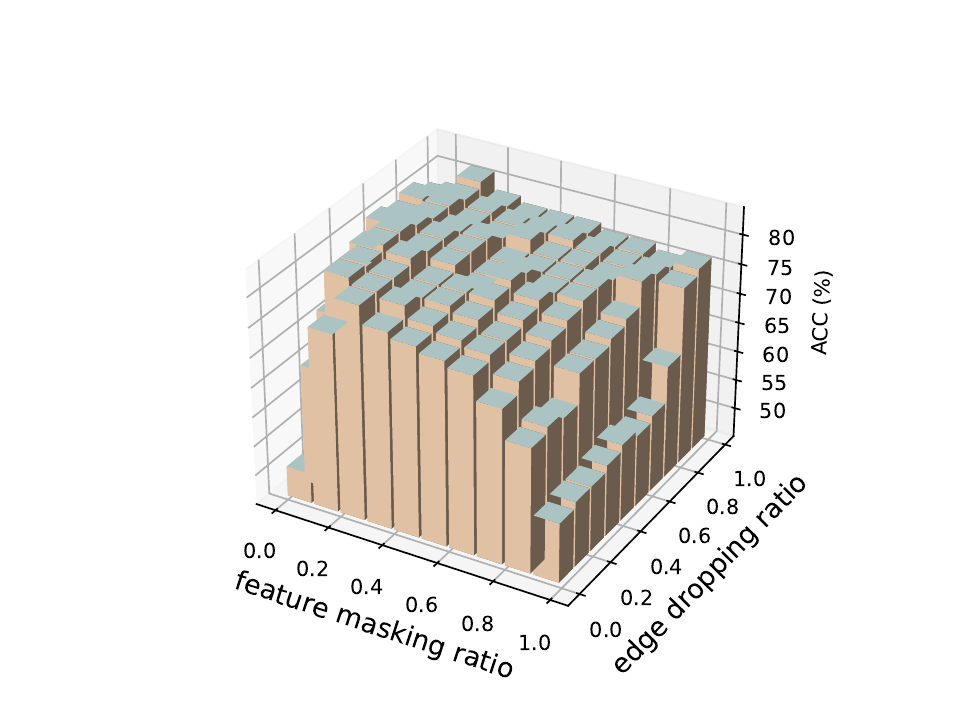}}   
    \subfigure[CiteSeer]{\includegraphics[width=0.32\linewidth]{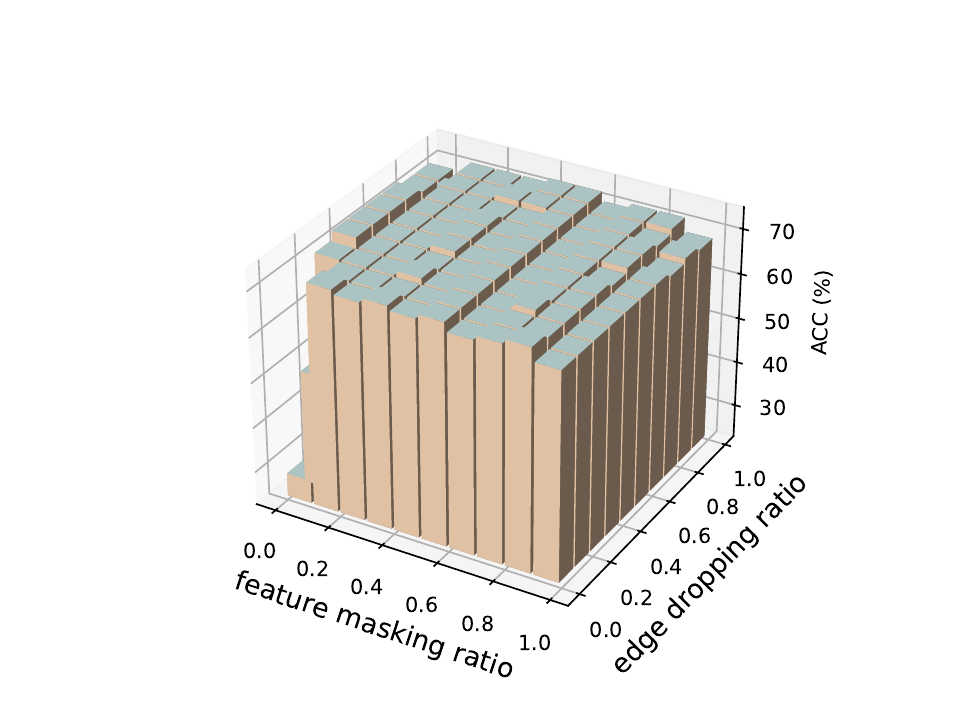}}  
    \subfigure[WikiCS]{\includegraphics[width=0.32\linewidth]{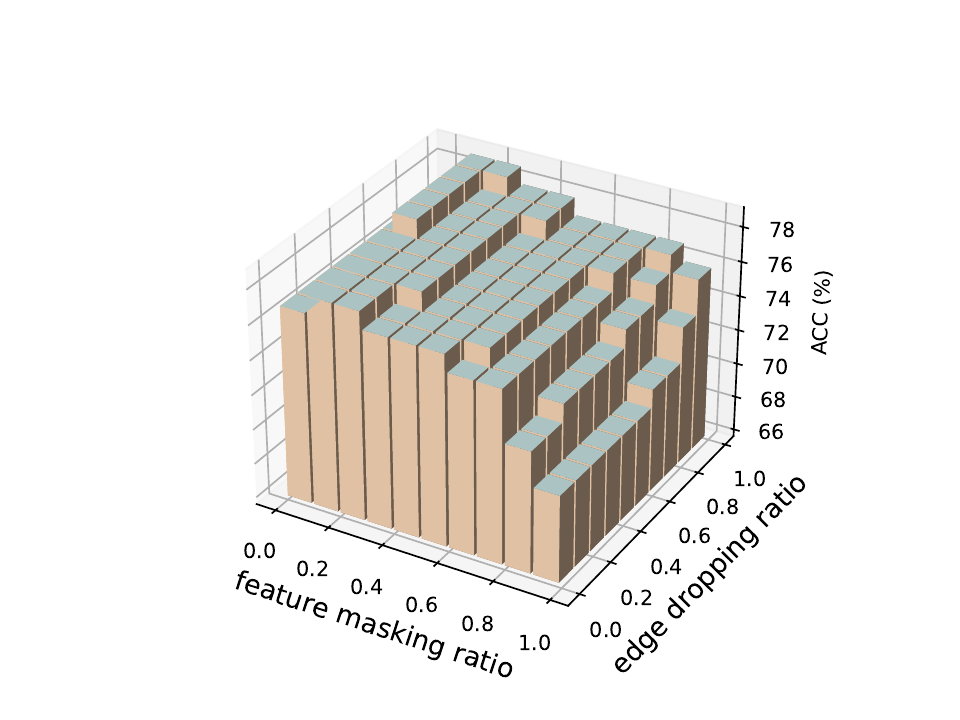}}
    \caption{Impact of the augmentation intensity.}
    \label{Fig: param_pdpm}
\end{figure}

% \begin{figure}[!ht]
% \centerline{\includegraphics[width=1.\linewidth]{images/dim.pdf}}
% \caption{Impact of the representation dimension. } \label{Fig: param_dim}
% \end{figure}

\section{Conclusion}
In this paper, we introduce a negative-free self-supervised objective, drawing inspiration from the fact that the normalized isotropic Gaussian distributed points are uniformly distributed on the unit hypersphere. Our new objective induces a simple and light model without reliance on negative pairs, a parameterized mutual information estimator, an additional projector or asymmetric architectures. Extensive experiments on node classification and node clustering across seven graph benchmarks illustrate that our model achieves competitive performance with fewer parameters, shorter training times, and lower memory costs compared to existing contrastive learning methods.

\noindent\textbf{Limitations and Future Work}: Our empirical studies were limited to static homogeneous graphs. Future work will need to explore extensions to other types of graphs, including heterogeneous and dynamic graphs. Additionally, our framework assumes that the learned representations follow a Gaussian distribution for the sake of simplicity. While our framework has been experimentally validated, it may not be the most suitable distribution in all contexts.

\section*{Acknowledgments}
This work is partially supported by the National Key Research and Development Program of China (2021YFB1715600), the National Natural Science Foundation of China (62306137). 
We would like to express our sincere gratitude to the reviewers, associate editor, and editor-in-chief for their invaluable time and effort dedicated to the evaluation of our manuscript.

% \appendix
% \section{Example Appendix Section}


%%       \cite{<label>} ==> [1]

%%
% Example citation, See \cite{lamport94}.

\bibliographystyle{elsarticle-num} 
\bibliography{references}




% %% For numbered reference style

%   \textit{\LaTeX: a document preparation system},
%   Addison Wesley, Massachusetts,

% \end{thebibliography}
\end{document}